\documentclass[11pt,a4paper,final]{article}
\usepackage[left=1in,right=1in,top=1in,bottom=1in]{geometry}
\usepackage{authblk}
\usepackage[utf8]{inputenc}
\usepackage{amsmath,amssymb,amsthm,mathtools,amsfonts}
\usepackage{xspace}
\usepackage{algorithm}
\usepackage{algorithmic}
\usepackage{mathtools}
\usepackage{nicefrac}
\usepackage{natbib}
\usepackage{hyperref}
\usepackage{cleveref}
\usepackage{enumitem}
\usepackage{color-edits}
\title{Federated Learning as a Network Effects Game}

\newtheorem{theorem}{Theorem}
\newtheorem{lemma}[theorem]{Lemma}
\newtheorem{corollary}[theorem]{Corollary}
\newtheorem{definition}[theorem]{Definition}
\newtheorem{remark}[theorem]{Remark}

\newcommand{\calN}{\mathcal{N}}

\newcommand{\calD}{\mathcal{D}}

\newcommand{\E}{\mathbb{E}}
\newcommand{\infer}{\text{inferred}}
\newcommand{\1}{\mathbf{1}}

\DeclareMathOperator*{\argmin}{arg\,min}

\newcommand{\ie}{{\em i.e.,~\xspace}}
\newcommand{\eg}{{\em e.g.,~\xspace}}
\newcommand{\etc}{{\em etc.}}

\newcommand{\bw}{\mathbf{w}}

\begin{document}
\author[1]{Shengyuan Hu}
\author[2]{Dung Daniel Ngo}
\author[1]{Shuran Zheng}
\author[1]{Virginia Smith}
\author[1]{Zhiwei Steven Wu}
\affil[1]{Carnegie Mellon University, \{shengyua, shuranzh, smithv, zstevenwu\}@cmu.edu}
\affil[2]{University of Minnesota, \{ngo00054\}@umn.edu}

\maketitle

\begin{abstract}
  Federated Learning (FL) aims to foster collaboration among a population of clients to improve the accuracy of machine learning without directly sharing local data. Although there has been rich literature on designing federated learning algorithms, most prior works implicitly assume that all clients are willing to participate in a FL scheme. 
  In practice, clients may not benefit from joining in FL, especially in light of potential costs related to issues such as privacy and computation. 
  In this work, we study the clients' incentives in federated learning to help the service provider design better solutions and ensure clients make better decisions. We are the first to model clients' behaviors in FL as a network effects game, where each client's benefit depends on other clients who also join the network. Using this setup we analyze the dynamics of clients' participation and characterize the equilibrium, where no client has incentives to alter their decision. Specifically, we show that dynamics in the population naturally converge to equilibrium without needing explicit interventions. Finally, we provide a cost-efficient payment scheme that incentivizes clients to reach a desired equilibrium when the initial network is empty.
\end{abstract}
\section{Introduction} 
Federated Learning (FL) is a distributed learning paradigm that enables a network of clients (\eg  mobile devices, hospitals) to jointly learn a model without sharing their private local data~\cite{mcmahan2017communication,kairouz2021advances}. In real-world applications, the performance of the model on a particular client may not be improved through federation due to heterogeneity in the data generated by federated networks \citep{li2021ditto}. Moreover, even if a particular client has improved model utility on their local data by joining FL training, they may suffer from costs induced by joining the federated network. For example, the communication or latency in FL may affect the quality of clients' user experience and the clients may suffer some privacy costs for sharing sensitive information. 
Understanding client incentives in FL is  crucial for the service provider to design better solutions and for each client to make better decisions. 

There is growing literature on incentive mechanisms in federated learning. Most of these works aim to study how to combine the data from all the clients to achieve a low error rate and ensure fairness for participating clients. However, these works are often based on a common assumption that all clients are willing participants who will join the federated network in order to achieve different learning goals (\eg minimize the sample complexity \citep{blum2021OneFO}, learn coalition structures \citep{modelsharing2020donahue}, \etc) In practice, rational clients can decide to either opt-in or opt-out of federated training based on their utility gain and cost as the network grows over time. The formal study of such a dynamic has not been captured in prior works studying incentives in federated learning. 

Understanding how incentives change in response to varying client participation is essential for the successful development of incentive mechanisms in federated learning.
In this work, we propose to model the clients' behaviors in FL through the lens of network effects games~\citep{katz1985network,shapiro2008information}. A network effects game models the benefit of each individual after aligning their behavior with the behaviors of other people in a coalition. This participation game closely relates to the client's behavior in FL where joining the federated training improves the utility compared to only training a model on local data. Traditionally, network effects games have not been looked at in the context of FL or more broadly, data sharing. Existing work in FL typically models client participation as a one-time game where the clients arrive and make decisions once before leaving. We instead focus on modeling the dynamic of client participation in a setting where (1) each client incurs a cost while joining the network, and  (2) each client is willing to join the network \textit{if and only if} their utility gain outweighs the cost. 
The goal of our network effects analysis is to find an equilibrium where no client has an incentive to alter their decision.

Based on such a formulation, we consider two settings: (1) a simple and stylized mean estimation problem (\Cref{sec:homogenenous-setting}) and (2) a general setting that requires an oracle reporting the true utility of being in a coalition to all clients (\Cref{sec:utility-oracle}). For each setting, we fully characterize the dynamics of the client participation game under the network effects model.  
A key feature of FL is that clients’ utility of joining depends on the number of participants/data points. Since different clients may have different utilities and costs, it may be easier to gradually incentivize participation over time. For example, we may first target the clients who have lower costs and are more willing to join, and their participation will increase other participants’ willingness to join, and so on. Based on this characterization, we design a cost-effective payment scheme that gradually targets a subset of clients over time to reach a desirable equilibrium in \Cref{sec:payment}.
We summarize \textbf{our contributions} below:
\begin{itemize}[leftmargin=*]
    \item We propose to model the clients' behaviors in a dynamic federated learning setting as a network effects game, in which clients' utilities depend on the number of participants. We want to find the self-fulfilling expectation equilibrium in FL, when the outcome of the clients' best-response strategy matches the shared expectation.  
    \item We characterize all self-fulfilling expectation equilibria of the federated network effects game in a mean estimation  problem and when there exists an oracle broadcasting the true utility of joining the coalition to all clients.
    \item We show that as clients continually best respond to public information, their dynamics naturally converge to an equilibrium where no clients have incentives to alter their decision.
    \item Finally, we provide a cost-efficient payment scheme that incentivizes clients to reach a desired equilibrium. Our result shows that for the dynamic to converge to a higher point, the server only needs to give payment until the coalition has reached a \emph{tipping point}.
\end{itemize}
\section{Related work}
\paragraph{Incentives in federated learning.}

Prior works that study incentive mechanisms in the context of federated learning typically aim to encourage clients to join in training in a one-time setting. Here we list the most relevant works, and defer readers to~\citep{tu2022incentive} for a more comprehensive discussion of work in the area.  Works on collaborative learning such as \citet{blum2017pac,haghtalab2022demand} describe the approach of handling heterogeneous data to learn a common concept by iteratively gathering more data. 
For the application of federated learning, \citet{modelsharing2020donahue, Donahue2021Optimality, Donahue2021modelfairness} analyzes the coalition structure for federated mean estimation and linear regression problems using hedonic game theory, and \citet{cho2022federate} proposes a new FL objective that aims at increasing the number of incentivized clients. Our method differs from the above works in that apart from utility gain, our work considers the cost of joining the federated network as a factor that can dynamically affect client behavior.

More related are the works of \citet{blum2021OneFO}, which analyzes a framework for incentive-aware data sharing in FL under the notion of envy, and  
\citet{karimireddy2022incentive}, which  introduces an accuracy-based algorithm to maximize the amount of data contributed by each client to learn their local mean parameter. While these two works share some similarities with ours in that they encode some form of joining cost, our method focuses on studying the dynamics of the population and analyzing the equilibrium when clients best respond to public information.

Finally, we note that the use of financial incentives as a means of motivating clients has been a topic of extensive research. Various economic approaches have been proposed to monetarily incentivize participation, e.g., via auctions/reversed auctions~\citep{kim2020incentive, thi2021incentive}, contract theory approaches~\cite{saputra2020federated}, public good approaches~\cite{tang2021publicgood}. Our work adds a new perspective to this area by studying how participation changes over time and considering payments only at tipping points of the dynamic. 

\paragraph{Network effects.}
Our work draws inspiration from literature on network effects games
\citep[e.g.,][]{katz1985network, shapiro2008information, easley2010networks},
where the utility of a client using a service depends on the number of other clients participating in the same service. This effect is often positive, as clients gain more utility from others joining the network. In our model, we consider accuracy improvement from having more data in a coalition as a natural incentive mechanism: As clients join the FL process, they expect to obtain a more accurate estimator from the combined data of all clients in the coalition. To the best of our knowledge, we are the first to model network effects in FL in this manner, and  our work is the first to study the behavior of rational clients in FL under the network effects model. In formulating and analyzing this dynamic perspective on client incentives, we hope that our work can lead to future study and development of FL incentive mechanisms that more closely mirror practical FL applications.
\section{Problem Formulation and Preliminaries}
In this section, we first introduce the setup of our federated learning problem. We then characterize the incentive and utility gain for each client from joining the federated learning process and their interaction with the server. Finally, we propose a general framework of federated learning as a dynamic network effects game. 

\subsection{Learning Problem Formulation}
We consider a federated learning setup where there are $M$ clients in the population who want to solve a common learning problem. For each client $i \in [M]$, we assume that they want to minimize their loss function given by $f_i(\bw) = \E_{\xi \sim \calD_i}[\ell(\bw, \xi)]$, where $\calD_i$ is the target data distribution of client $i$, and $\ell(\bw, \xi)$ is the loss function for a model $\bw$ for data sample $\xi$. There are two possibilities: all clients share a target distribution $\calD_i = \calD$, \eg global mean estimation, and when clients have different target distribution $\calD_i$, \eg personalized local mean estimation. In practice, each client only has access to its local training dataset with $n_i$ data point sampled i.i.d from data distribution $\calD_i$. 

\paragraph{Client Incentive and Utility Gain.} The goal of each client $i$ is to find a model $\bw^*_i$ that minimize its true loss function, \ie $\bw^*_i \coloneqq \arg\min_{\bw} f_i(\bw)$. We assume that each client will perform local training on their own dataset to obtain a local model $\hat{\bw}_i$. Since the local dataset has a small size, the local model $\hat{\bw}_i$ might not generalize well to the target distribution $\calD_i$. Hence, a client would consider participating in the federated learning process if the federated model gives a better generalization performance compared to its local model. We define the expected utility gain of joining a coalition $S$ for a client $i$ as the difference between the true loss using the local model $\hat{\bw}_i$ and the federated model $\bw_S$ on the target distribution $\calD_i$:
\begin{equation*}
    U_i = \E_{\calD_i}[f_i(\bw_S) - f_i(\hat{\bw_i})]
\end{equation*}

\paragraph{Bayesian priors.}
We assume each client $i$ has a Bayesian belief about a model's true loss as a function of the number of samples used in the model. 
That is, given the local sample size and the coalition size, each client $i$ can infer the losses $f_i(\hat{\bw}_i)$ and $f_i(\bw_S)$. Hence, given the shared expectation $K$, the clients have some belief about their expected utility gain for joining the coalition. When the clients' belief exactly matches the actual utility gain of joining, this belief is equivalent to the clients observing the true utility from an oracle. 

\paragraph{Personal cost.} If a client $i$ in the population decides to join the coalition, they will incur some fixed cost $c_i > 0$. This cost $c_i$ contains the enrollment and communication cost of joining the coalition. The cost may also represent the privacy loss for the client $i$ from revealing information about their local data to the server. For example, a server can charge each participating hospital a fixed fee for using the shared model. At each time step, each hospital needs to communicate with the server to update its knowledge base and decide to either opt-in or opt-out. Finally, due to the sensitive nature of medical data and how each individual model is trained, the hospitals will also incur some privacy loss for sharing their model with the server. 

\paragraph{Interaction between clients and server.}
Given a coalition $S$, a client has to decide whether to join the coalition or use only their local data. We assume that if the utility gain from joining the coalition is less than the personal costs for a client $i \in [M]$, then client $i$ would not join the coalition. That is, given a coalition $S$ and a client $i \in [M]$ with cost $c_i$ where $U_i \geq c_i$, then client $i$ would opt in the federated learning process. Otherwise, if the utility gain is less than that of the personal cost $c_i$, then client $i$ would opt out of the federated learning process.

\subsection{FL as Dynamic Network Effects Game}
We model the dynamic of the population in federated learning as a \emph{network effects} game. Whenever all clients in the population share an expectation that $K$ samples are in the coalition, each client $i \in [M]$ can calculate their expected utility gain as a function of $K$, \ie $U_i = U_i(K)$ and compare it to their cost $c_i$. After a coalition is formed, all clients in the population can observe the actual number of samples $h(K)$ in the coalition and update their shared expectations for the next time step. Formally, we describe the participation decision of the clients at each time step in \Cref{alg:dynamic}. We define the \emph{realization mapping} between the shared expectation $K$ and the actual outcome in the coalition as $h: \mathbb{Z}^{+} \rightarrow \mathbb{Z}^{+}$, \ie if the shared expectation is $K$, then the actual outcome in the coalition is $h(K)$.

\begin{figure}[ht]
    \centering
    \includegraphics[width=0.5\linewidth]{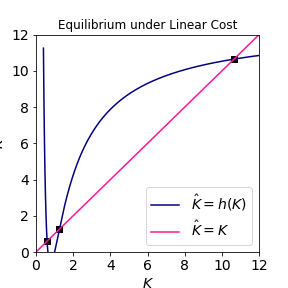}
    \caption{If the shared expectation is $K$, then $\hat{K} = h(K)$ samples will be in the coalition. When the curve $\hat{K} = h(K)$ crosses the line $\hat{K} = K$, we have self-fulfilling expectation equilibria. In our analysis, we only consider integer value equilibria.}
    \label{fig:realization-function}
\end{figure}

As a result, clients can compute their utility gain according to the new expectation and decide to either stay/join the coalition or leave. For example, if there are more samples in the coalition than expected, \ie $h(K) \geq K$ then the utility gain from joining increases. As a result, some new clients now have positive expected utility gains and are willing to join the coalition. On the other hand, if there are fewer samples in the coalition than expected, the expected utility gain would decrease. Some clients who were previously in the coalition would leave as their utility gain is lower than their cost, while other clients have no incentives to join. We aim to study the dynamic of the population and characterize a suitable equilibrium under these behaviors. 

\begin{algorithm}[tb]
   \caption{Dynamic of client participation}
   \label{alg:dynamic}
\begin{algorithmic}
   \STATE {\bfseries Input:} Clients $1, \cdots, M$ with personal cost $c_1, \cdots, c_M$ and number of local samples $n_1, \cdots, n_M$, respectively; coalition $S$.
   \STATE At $t=0$, all clients form a shared expectation $K$ on how many samples will be in the coalition $S$.
   \FOR{time step $t = 1, \cdots, T$}
   \FOR{client $i=1,\cdots,M$}
   \STATE Client compute expected utility gain $U_i(K)$ based on shared expectation $K$.
   \IF{$U_i(K) \geq c_i$}
        \STATE Client $i$ will join coalition $S$
   \ELSE
        \STATE Client $i$ will leave coalition $S$
   \ENDIF
   \ENDFOR
   \STATE Update $K = \sum_{j \in S} n_j$ to be the new shared expectation for the next time step. 
   \ENDFOR
\end{algorithmic}
\end{algorithm}

If the clients' shared expectation is perfect, we can define an equilibrium notion. If all clients form a shared expectation that there are $K$ samples in the coalition, and if each of them decides to either join or leave the coalition based on this expectation, then the actual number of samples in the coalition is $K$. We call this a \emph{self-fulfilling expectation} for the number of samples $K$: 
\begin{definition}[Self-Fulfilling Expectation Equilibrium] Consider a population where all clients share an expectation that there are $K$ samples in the coalition. If the actual number of samples in the coalition is $h(K) = K$, then $K$ is called a self-fulfilling expectation equilibrium.  
\end{definition}

By its definition, if we know function $h(K)$, we can find the self-fulfilling equilibria by locating the integral crossing points of $h(K)$ and the identity function $f(K) = K$. But the question is: do they always cross at integer points? We show by the following theorem that whenever we have $h(K)\ge K$, we are guaranteed to find an integral equilibrium $h(K^*) = K^*$ to the right of $K$.

As we will show in our analysis, we can identify self-fulfilling expectation equilibria in the dynamic as long as the realization mapping $h(\cdot)$ is monotonically non-decreasing for some stylized learning problems. 
\begin{theorem}[Existence of Self-Fulfilling Expectation Equilibrium] In FL under the dynamic network effects game model, if the realization mapping $h(.)$ is monotonically non-decreasing, for any positive $K$ with $h(K) \geq K$, we can find a self-fulfilling expectation equilibrium $h(K^*) = K^*$ to the right of $K$, \ie $K^* \geq K$. For any point in between $K' \in (K, K^*)$, we have $h(K') > K'$.  
\label{thm:convergence}
\end{theorem}

Characterizing self-fulfilling expectation equilibria is crucial to the server and policymakers who wish to understand and influence where the dynamic converges. In \Cref{sec:homogenenous-setting} and \ref{sec:utility-oracle}, we show that \Cref{thm:convergence} is true for two stylized settings of mean estimation and utility oracle. 
In \Cref{sec:payment}, we show that the server use properties of the equilibria to guarantee that the final coalition reaches the maximum size equilibrium by paying a subset of clients. 

\paragraph{Stable equilibrium.}
When the population is at equilibrium, it could still be susceptible to small perturbations when a client decides to join or leave the coalition. Consider an extreme instance where the highest cost of joining for a client is the same as the utility gain at every value $K$, i.e. $\forall K \in [M]: c_K = U_K$. Then, by definition, every value $K$ is a self-fulfilling expectation equilibrium. Hence, at each value $K \in [M-1]$, new clients still have incentives to join the coalition as their cost of joining is exactly the same as the utility gain. In this instance, subject to small perturbation, the dynamic will converge to either $0$, where no client is in the coalition, or $M$, where every client is in the coalition. These equilibrium values where no client has incentives to join or leave are called stable equilibrium. Formally, we define stable equilibrium as:
\begin{definition}[Stable Equilibrium \citep{jackson2007network}] 
\label{def:stable-equilibrium}
An equilibrium $K^*$ is stable if there exists $\epsilon' > 0$ such that $h(K^* - \epsilon) > K^* - \epsilon$ and $h(K^* + \epsilon) < K^* + \epsilon$ for all $\epsilon' > \epsilon > 0$. An equilibrium $K$ is unstable (tipping point) if there exists $\epsilon' > 0$ such that $h(K-\epsilon) < K-\epsilon$ and $h(K+\epsilon) > K+\epsilon$ for all $\epsilon' > \epsilon > 0$.
\end{definition}
In the example above, all values $K \in [1, K-1]$ are the tipping points. With this definition, we can analyze the dynamics of the population. Informally, if the initial expectation is not a stable equilibrium, then the clients can continually update their shared expectations and decide to either join or leave the coalition. This decision-making process ends when the shared expectation is at a stable equilibrium, where with a small perturbation to the coalition, there is no incentive for any client to change their decision.
\begin{theorem}[Convergence to stable equilibrium]
In FL under the dynamic network effects game model, if clients share an expectation $K$ with $h(K) > K$ then the coalition size will grow in the next time step. Otherwise, if $h(K) < K$, then the coalition size will shrink in the next time step. If the shared expectation is between a tipping point $K'$ and a stable equilibrium $K^*$, then the dynamic will converge to the stable equilibrium $K^*$. 
\label{thm:convergence-stable}
\end{theorem}
In our analysis, we assume that there always exists one trivial self-fulfilling expectation at $0$, \ie if all clients expect that no one would join then there are no samples in the coalition and $U_i(0) = 0$. Since the cost of joining for a client $i$ is $c_i > 0$, no client would join the coalition. In our analysis, we always assume the initial expectation is positive unless specified otherwise. We write equilibrium as shorthand for self-fulfilling expectation equilibrium.

\begin{remark}[Stable equilibrium is a Nash equilibrium]
First, we observe that not every self-fulfilling equilibrium is a Nash equilibrium. Consider the example where the cost of joining is equal to the utility gain at every value of $K$. Then at some tipping point $K$, given the decision of every other client, a client $K+1$ who is not in the coalition can potentially join and increase their utility. Hence, the tipping point $K$ is not a Nash equilibrium.  

On the other hand, a stable equilibrium $K^*$ is a Nash equilibrium. By definition, when the shared expectation is $K^*$, no client in the population has incentives to change their decision. Hence, the set of stable equilibrium is an intersection of Nash equilibria and self-fulfilling equilibria. 
\end{remark}

\section{Network Effects game in a Mean Estimation Problem}
\label{sec:homogenenous-setting}

In this section, we investigate how the dynamic behaves in a stylized global mean estimation problem. First, we formally describe the setup when there all clients in the population have the same number of local samples. Then, we show that \Cref{thm:convergence} is true in this setting with some simple cost functions. Finally, we show that the dynamic naturally converges to a stable equilibrium where no clients have further incentive to change their decisions.
\paragraph{Mean Estimation Problem Setup.}
There are $M$ clients in the population who want to solve a common mean estimation problem. Each client $i$ has a fixed number of samples $n_i$ and a cost $c_i$. First, each client $i$ draws their mean parameter $\mu_i$ i.i.d from a common prior with mean $\theta$: $\mu_i \sim \calN(\theta, \sigma_{\theta}^2)$. W.L.O.G, we assume that client $i$ draws their samples i.i.d from a unit variant Gaussian with local mean $\mu_i$: $X_i \sim \calN(\mu_i, 1)$. Formally, the clients want to find an estimator $w$ such that the mean squared error (MSE) $\E[(\theta - w)^2]$ is minimized. We first focus our analysis on the case where all clients in the population have the same number of local data to highlight the dynamic behavior. We defer the discussion of the general case where clients have different numbers of local samples to the appendix.

For a client $i$, if they do not join the coalition, their estimator is simply the empirical average over their local data samples $\bw_i = \nicefrac{1}{n_i} \sum_{j=1}^{n_i} x_{i,j}$. Otherwise, if the client decides to join a coalition $S$, we assume they will send their local estimator $w_i$ to the server for aggregation. The server then computes a coalition estimator, which we assume is a weighted average over the local estimators of participating clients in $S$: $\bw_S = \nicefrac{1}{N_S} \sum_{j \in S} w_j \cdot n_j$ where $N_S$ is the total number of samples in the coalition.

Since all clients have the same number of local samples $n_i = n$ and the expected number of clients in the coalition is $K$, the utility gain function for a client $i$ is:
\begin{equation}
    U_i = U(K,n) = \frac{K-1}{Kn} + \frac{3K^2 - 5K + 2}{K^2}  \sigma^2_{\theta}
\label{eq:homogeneous-utility-gain}
\end{equation}
When all clients have the same number of local samples, we slightly abuse the notation and write $K$ as the expected number of clients in the coalition instead of the expected number of samples. Similarly, $h(K)$ represents the actual number of clients in the coalition in this setting.

Since all clients have the same number of local samples, the equilibrium only needs to consider the number of clients in the coalition. Our first result here is a sufficient condition for such an equilibrium to exist. Next, we explore some desirable properties of self-fulfilling expectation equilibrium, \ie characterizing a subset of equilibrium that is stable where no clients in the population have incentives to change their decision. We formally show that under the network effect model, all dynamics will converge to a stable self-fulfilling expectation equilibrium. Finally, we provide sufficient conditions for such a stable equilibrium to exist. 


\subsection{Existence of a Self-fulfilling Expectation Equilibrium} 
\paragraph{Realization mapping $h(\cdot)$.}
Since all clients have the same number of local samples, W.L.O.G, we can order them according to their personal cost in ascending order. Moreover, we assume that the cost of joining for a client $i$ is a continuous positive monotonically non-decreasing function of their index $i$, \ie $c_i = c(i)$. In our analysis, we only look at the cost function for integer value $i$. If all clients share a common belief that $K$ clients will join the coalition, then client $i$ would join only if $U(K,n) \geq c(i)$. Hence, if any client would join at all, then the set of clients joining would be between $1$ and $\hat{K}$, where 
$$h(K) = \inf\{ x \in \mathbb{R}: U(K,n) \geq c(i)\}$$
Hence, we can find the equilibrium by solving $U(K,n) = c(h(K))$.
Fixing the number of local samples $n$, this equation is equivalent to solving 
$$h(K) = c^{-1}(U(K,n))$$
Formally, we define the mapping between the expectation $K$ and the actual number of clients in the coalition $\hat{K}$ as $h: [M] \rightarrow [M]$, where 
$$h(x) = c^{-1}(U(x))$$
Hence, if we exactly know the cost function $c(\cdot)$ and the expected utility gain $U(K,n)$, we can find an equilibrium in the population. 
Since $c(i)$ is monotone non-decreasing, its inverse function $h(\cdot)$ is also monotone non-decreasing. 

Moreover, we prove that there always exists an equilibrium under some mild assumption on shared expectation $K$ and the actual number of clients in the coalition $h(K)$. Informally, if the shared expectation $K$ is pessimistic compared to the actual outcome $h(K)$, then we can always find an equilibrium. The clients would update their expectations and join the coalition due to a new higher utility gain. However, as the $h(\cdot)$ function is non-decreasing, there exist some clients whose utility gain is less than the cost of joining and they do not have incentives to join the coalition. 
\begin{corollary}
In the mean estimation problem where all clients have the same number of local samples, for any positive $K>0$ with $h(K)\ge K$, we can find a self-fulfilling expectation equilibrium $h(K^*) = K^*$ to the right of $K$ with $K^*\ge K$; and for any point in between $K'\in (K, K^*)$, we have $h(K')>K'$. 
\label{thm:sufficient-condition-homogeneous}
\end{corollary}
This result can also be interpreted as every time the actual number of clients in the coalition $\hat{K}$ equals or exceeds the expectation $K$, we must have an equilibrium. Under this assumption, we can derive a sufficient condition on the cost function $c(\cdot)$ so there exists an equilibrium. Note that $h(K)$ is a composition of the cost function $c(i)$ and the utility function $U(K,n)$. Since the cost function $c(\cdot)$ is monotonically non-decreasing, we can apply $c(\cdot)$ to both sides of the sufficient condition in \Cref{thm:sufficient-condition-homogeneous} and obtain an equivalent condition: an equilibrium exists as long as there exists an expectation $K$ such that $U(K,n) \geq c(K)$.

\subsection{Convergence to Stable Equilibrium}

We derive a sufficient condition for a population to have a stable equilibrium. 
Informally, a stable equilibrium exists when all clients are in the coalition or if the cost of joining for client $K^*+1$ is larger than the utility gain from joining.
\begin{corollary}
In the mean estimation problem when all clients have the same number of local samples, a stable equilibrium $h(K^*) = K^*$ exists as long as there exists an equilibrium $K^*$ with $h(K^*+1) < K^*+1$ or $K^* = M$.
\label{thm:sufficient-condition-stable-homogeneous}
\end{corollary}

Finally, we show that \Cref{thm:convergence-stable} applies and the dynamic would naturally converge to the stable equilibrium. 
\begin{corollary}[Convergence to stable equilibrium in mean estimation setting]
In the mean estimation problem when all clients have the same number of local samples, if an expectation $K$ is between a tipping point $K'$ and a stable equilibrium $K^*$, then the dynamic will converge to $K^*$. 
\label{cor:convergence-homogeneous}
\end{corollary} 

\subsection{Extension: when clients have different numbers of local samples}
In the previous analysis, we studied a stylized setting where all clients have the same number of local samples. 
When clients have a different number of samples, this analysis does not easily generalize since we cannot directly infer the utility gain from the expected number of samples $K$. Notably, in \Cref{eq:utility-gain}, the cross term $\sum_{i\neq j}^{j\in S} n_j^2$ requires the clients to have knowledge of the coalition composition to calculate the expected utility gain. Hence, we need an additional step where the clients form a shared expectation of the utility gain before making their decision. Our first result here is a sufficient condition for an equilibrium to exist. Then, we proceed to show that the dynamic would converge to a stable equilibrium.

\paragraph{Inferring Coalition from Shared Expectation}
First, we investigate how the clients infer the coalition composition based on the expected number of samples $K$ in the coalition. Assuming that the number of local samples $n_i$ and the personal cost $c_i$ for each client $i \in [M]$ are public information, the server can separate the utility gain function into two parts: a fixed gain using data from other clients in the coalition and the additional gain by having their personal data in the coalition. That is, for a client $i$ with $n_i$ local samples, the utility gain for joining the coalition is
\begin{align}
    U_i(\{n_j\}_{j\in S})
    = \underbrace{\frac{1}{-N_S} - \left( \frac{\sum_{j \in S} n^2_j}{N^2_S} - 3 \right) \sigma^2_\theta}_{\text{fixed gain}} + \underbrace{\frac{1}{n_i} + \left( \frac{2n^2_i}{N^2_S} - \frac{4n_i}{N_S} \right) \sigma^2_\theta}_{\text{additional gain}}
    \label{eq:heterogeneous-utility-split}
\end{align}

Let $z_i = \frac{1}{n_i} + \left( \frac{2n^2_i}{N^2_S} - \frac{4n_i}{N_S} \right) \sigma^2_\theta - c_i$ be the difference between the additional gain from joining the coalition and the cost for client $i$. Since the first part of the utility gain function is fixed for all clients, we can index the clients in descending order of $z_i$. Then, the population can infer a set of clients who will join the coalition according to this ordering. Formally, this inferred set of clients contains the first $t$ clients in the population whose total number of samples is larger or equal to $K$, \ie $t = \argmin_x \sum_{j=1}^x n_j \geq K$. Then, we can define $S_I(K) = \{1, \dots, t\}$ to be the inferred coalition with the clients indexed in descending order of $z_i$. 

In the definition of $t$, the expected number of samples $K$ may not exactly match the inferred number of samples in the coalition. 
We make the following \emph{tie-breaking} assumption: even if the expectation $K$ does not exactly matches the inferred number of samples in the coalition, the population can still infer that the coalition contains clients with an index from $1$ to $t$ as these clients have the most incentives to join the coalition. These clients can compute their utility $U_i(K) = U_i(\{n_j\}_{j\in S_I(K)})$ and decide to join or not by comparing it with their cost $c_i$. We define the realization mapping $h(K)$ as the sum over samples of clients in the inferred coalition whose utility gain is larger or equal to their cost, \ie 
\begin{align}
    h(K) &= \sum_{j=1}^t n_j \cdot \1\{U_j(K) \geq c_j\} \nonumber \\
    &= \sum_{j=1}^t n_j \cdot \1\left\{ \frac{-1}{N_{S_I(K)}} - \left(\frac{\sum_{j=1}^t n_j^2}{N_{S_I(K)})^2} \right)\sigma^2_\theta - z_i \geq c_j\right\} 
\label{eq:mapping-heterogeneous}
\end{align}
where $N_{S_I(K)} = \sum_{j \in S_I}^t n_j$ is the total number of samples in the inferred coalition. With this inferred coalition, we can analyze the self-fulfilling expectation equilibrium in this setting. The rest of the analysis in this section follows the same proof technique as the previous setting where all clients have the same number of local samples. The first result here is that we can prove \Cref{thm:convergence} in the mean estimation problem when clients have different numbers of local samples:
\begin{corollary} In the mean estimation problem where clients have different numbers of local samples, the realization mapping is $h(K) = \sum_{j=1}^t n_j \cdot \1\{ U_j \geq c_j\}$ where $U_j(K)$ is defined in \Cref{eq:mapping-heterogeneous}.  
For any positive $K>0$ with $h(K)\ge K$, we can find a self-fulfilling expectation equilibrium $h(K^*) = K^*$ to the right of $K$ with $K^*\ge K$; and for any point in between $K'\in (K, K^*)$, we have $h(K')>K'$.
\label{thm:sufficient-condition-heterogeneous}
\end{corollary}

Finally, we also show that \Cref{thm:convergence-stable}  also apply to this setting and the dynamic will converge to a stable equilibrium:
\begin{corollary}[Convergence to stable equilibrium] In the mean estimation problem where clients have different numbers of local samples, if an expectation $K$ is initialized between a tipping point $K'$ and a stable equilibrium $K^*$, then the dynamic will converge to $K^*$. 
\label{cor:convergence-heterogeneous}
\end{corollary}

\section{Network Effects in FL with Utility Oracle}
\label{sec:utility-oracle}
In the previous analysis, we rely on the assumption that all clients are trying to solve a common global mean estimation problem to derive the exact utility gain function. The clients then make their opt-in/opt-out decisions based on this expected utility gain value. How would the clients' behavior change if the server and the clients do not know how to calculate the utility gain, but instead observe it from some oracle for some general learning problem?  In this section, we will study this question assuming the underlying utility gain function is monotonically non-decreasing with regard to the total number of samples in the coalition. Since the utility gain function \Cref{eq:homogeneous-utility-gain} in the homogeneous setting satisfies our monotonicity assumption and we are considering a more general learning problem, the utility oracle setting is strictly more general than the mean estimation setting. First, we formalize our definition of self-fulfilling expectation equilibrium with utility oracle. Then, we show that the dynamic will always converge to a self-fulfilling expectation equilibrium with proof by induction. 

\subsection{Client/Server Interaction with a Utility Oracle} 
Consider a population of $M$ clients who want to solve a common learning problem over a time horizon $T>0$. At $t=0$, there exists a non-empty arbitrary coalition $S^{(0)}$ consisting of some clients in the population. Suppose at each time step $t$, there is an oracle that broadcast a message to all the clients about the true utility of joining the coalition based on the previous time step $t-1$ regardless of their current participation status. That is, given the coalition $S^{(t)}$ on time step $t$, a client $i \in [M]$ knows $U_i^{(t)} = u(\{n_j\}_{j \in S^{(t)}}, i)$ if client $i$ is currently in the coalition. Otherwise, if $i$ is not in the coalition, then client $i$ can know $U_i^{(t)} = u(\{ n_j\}_{j \in S^{(t)} \cup \{i\}}, i)$. With this observed utility, at the next time step $t+1$, clients $i \in [M]$ would opt-in the coalition if $U_i^{(t)} \geq c_i$ and  leave the coalition otherwise.  

In our analysis, we assume that this observed oracle utility only depends on the size of the coalition and the opt-in/opt-out decisions of other clients in the coalition. This assumption aligns with how the utility gain function was defined for the homogeneous setting in \Cref{eq:homogeneous-utility-gain}. With this interaction protocol, we note that there is no longer a need for the population to form an expectation over the size of the coalition at each time step. Instead, clients $i \in [M]$ would use the oracle utility of the previous time step $t-1$ as the expected utility for time step $t$, \ie $\E[U_i^{(t)}] = U_i^{(t-1)}$. We rewrite the definition of self-fulfilling expectation equilibrium for the oracle utility setting as:

\begin{definition}[Self-fulfilling expectation equilibrium with oracle utility] Consider a population in which all clients observe oracle utility at time step $t$ with $K$ samples in the coalition. If the actual number of samples in the coalition at time step $t+1$ is also $K$, then $K$ is called a self-fulfilling expectation equilibrium. 
\end{definition}
That is, at every time step, the clients can best respond to the observed oracle utility and change their participation status. Since clients are behaving in a myopic way with limited information, if the coalition stays the same in two consecutive time steps, then it would also stay the same for the rest of the time horizon. With our assumption that the utility gain only depends on the number of samples in the coalition, this condition is equivalent to the coalition having the same size in two consecutive time steps. 
\subsection{Existence of Self-fulfilling Expectation Equilibrium}
We prove that there always exists an equilibrium if the observed utility gain for all clients in the population satisfies the monotonicity assumption. Informally, if the coalition grows in size in the first few time steps, all clients in the population would observe an increased utility gain from participating in the federated learning process. Hence, new clients $j \notin S$ with $U_i^{(t)} \geq c_i$ would join the coalition while opted-in clients $i \in S$ do not have the incentive to leave. Since the size of the coalition does not decrease at each time step, this effect would cascade until the coalition has reached an equilibrium, where the extra utility from having more clients is not sufficient for new clients to join the coalition. On the other hand, if the initial coalition contains some clients with very high costs who would leave in the first few time steps, then all clients would observe a decreased utility gain in future time steps. Hence, clients who have not already joined the coalition would no longer want to participate in the federated learning process. Moreover, clients who joined in previous time steps might leave once their decreased utility from a smaller coalition becomes lower than their personal cost. Therefore, the total number of samples in the coalition would shrink until we have reached a self-fulfilling expectation equilibrium. 

\begin{theorem} Given a population where all clients observe the utility oracle regardless of their participation status, the dynamic will converge to a self-fulfilling expectation equilibrium. More specifically, there exists a time step $t$, where $S^{(t)}$ is a self-fulfilling expectation equilibrium and $S^{(t+1)} = S^{(t)}$ thereafter.
\label{thm:utility-oracle-convergence}
\end{theorem}

This result highlights the dynamic behavior of the population under the network effects model where clients are behaving in a myopic way with limited information. Under a mild assumption of monotonically increasing utility w.r.t the number of samples in the coalition and with any non-empty initial coalition, the dynamic always converges to a self-fulfilling expectation equilibrium. The monotonicity assumption is often justifiable---since all local samples are drawn i.i.d, having more samples in the coalition would intuitively produce a model with better accuracy in the mean estimation task. While the assumption of utility oracle is often not realistic, its inclusion removes the need for a closed-form expression for the utility gain function, which was previously needed in our analysis. As a result, our analysis here can also be applied to other learning problems that are not mean estimation as long as the utility gain for each client satisfies the monotonicity assumption.   

\section{Network Effects in FL with Payment}
\label{sec:payment}
In the previous analysis, we implicitly assumed that the initial shared expectation and the initial coalition are non-empty. This assumption is made to prevent the dynamic from getting stuck at the trivial stable self-fulfilling expectation equilibrium of $0$. Without any intervention from the server, no client would voluntarily join the coalition. In this section, we look at a possible solution to this problem as the server can start the dynamic from $0$ by offering external incentives to a subset of clients. We begin with modeling a mechanism design problem where at every time step, the server wants to find the smallest payment that can kick start a dynamic convergence to the largest equilibrium. We then provide a characterization of the optimal payment mechanism for the server that relies on the number of tipping points and stable equilibria in the dynamic.

\subsection{Payment Mechanism as an Optimization Problem}
First, we consider the same setting where a population of clients wants to learn a common learning problem but the server can offer to subsidize the personal cost of some clients by waiving their enrollment fee for using the server's aggregated model. Moreover, the server also provides extra payment to offset the privacy loss from sharing the client's personal model trained on their own data samples. That is, a client $i \in [M]$ with personal cost $c_i$ can have their entire cost subsidized by the server such that they can join the coalition for 'free'. We assume that at every time step, the server can broadcast a take-it-or-leave-it payment message to a subset of clients. For the server, if it pays $K$ clients to join the coalition, then the total amount of payment $P$ is the sum of the personal cost for these $K$ clients.

The mechanism design problem for the server can be seen as an optimization problem. Given the personal costs of all clients, we want to find the smallest payment such that the dynamic converges to the largest equilibrium. That is, for some small perturbation $\epsilon > 0$, all clients $j$ not in the coalition at the time horizon $T$ will not join the coalition even if there are $\epsilon$ more samples in the coalition, and no clients in $S$ have incentives to leave. Formally, we have:
\begin{align*}
    &\min P \\
    \text{s.t.} \quad &\begin{cases}\forall j \notin S, U_j^{(T)}(N_S + \epsilon) < c_j \\
    \forall i \in S, U_i^{(T)}(N_S) > c_i \end{cases}
\end{align*}

Alternatively, if the server has a payment budget $B>0$, then the optimization problem becomes finding the highest equilibrium point that the dynamic can possibly reach with a total payment of less than $B$. Due to the stability property of some equilibria in the dynamic, the server would approach these two optimization problems with unlimited and limited budgets with the same strategy. In intuition, the optimal strategy involves paying a specific subset of clients at each time step to move the dynamic to a higher self-fulfilling expectation equilibrium. Since the number of paid clients is fixed at each time step, the server's payment strategy is the same for both scenarios. From this point, we focus our analysis on the unlimited budget setting. 

\subsection{Payment in Mean Estimation Setting}
First, we look at the mean estimation problem where all clients have the same number of local samples. If the initial shared expectation is $0$, then we know that the expected utility of joining the coalition is also $0$. Hence, no client would voluntarily join the coalition. From our analysis in \Cref{sec:homogenenous-setting}, if the expectation is between a tipping point and a stable equilibrium then the dynamic would converge to the stable equilibrium. Hence, the server can minimize their total payment by first broadcasting a payment message to the first $K_0$ clients with the lowest personal cost, where $K_0$ is the smallest tipping point. If the initial coalition has fewer than $K_0$ clients, then it would quickly collapse back to the trivial stable equilibrium of $0$. Once the coalition has reached a tipping point, the server can help move the dynamic toward a higher tipping point simply by getting one new client to join. This new client would increase the expected utility gain of being in the coalition enough that other clients would also consider joining the coalition. This payment process would stop once the coalition has reached the highest equilibrium. We formalize the optimal payment strategy for the homogeneous setting as:

\begin{theorem}[Incentive with payment in the mean estimation setting] Consider a population in which all clients have the same number of local samples $n$ and a server that can repeatedly send payments to incentivize new clients at every round. The server would move the dynamic toward the highest equilibrium with the minimum total payment by sending a payment message to one new client with the smallest personal cost at every tipping point or to $K_h$ new clients at every stable equilibrium, where $K_h$ is the difference in the number of clients between the $h$-th stable equilibrium and its closest right tipping point. 

\label{thm:payment-homogeneous}
\end{theorem}
We observe that the server can only save on their payment by relying on the stability property of some self-fulfilling expectation equilibrium. Hence, if a dynamic has an equilibrium near $0$ and few tipping points, then a server following the strategy in \Cref{thm:payment-homogeneous} would have to spend much less than the naive strategy of fully paying every client to be in the coalition. On the other hand, if the dynamic has a lot of tipping points or if the coalition requires a lot of new clients to move from a stable equilibrium to a higher tipping point, then following the payment strategy above does not vastly improve upon the naive strategy of paying every client. Specifically, we revisit the example where every value $K$ is a self-fulfilling expectation equilibrium. 

\vspace{-.1in}
\paragraph{Example: equilibrium at every point.} Revisit the example where the cost of joining is the same as the utility gain at every value $K$, \ie $\forall K \in [M]: c(K) = U(K,n)$. By definition, every value $K$ is a self-fulfilling expectation equilibrium. If the initial expectation is $0$, the server has to first pay the client with the lowest cost. At this point, the coalition has one client and has reached a tipping point. To move the dynamic towards a higher equilibrium, the server sends out another payment to the next client with the second lowest cost. This process stops when all clients have joined the coalition. Since the server has to pay for a new client every time, the total  payment is $P = \sum_{i=1}^M c_i$. 

\begin{remark} Since the clients who join the coalition also receive some utility from the joint model, we can also consider a more efficient payment scheme where the server only pays for the difference between a client's cost and their expected utility gain. When all clients have the same number of local samples, they also have the same expected utility gain so the server would minimize their total payment by selecting clients with smaller costs first. This payment strategy is similar to our original strategy when the server has to pay for a client's entire cost. 
\end{remark}

\subsection{Payment in Utility Oracle Setting}
We can generalize the payment structure above in the more general utility oracle setting when the initial coalition is empty. Similar to the mean estimation setting, at each time step, the server needs to rely on the stability property of self-fulfilling expectation equilibrium to reduce their total payment amount. That is, at each time step, the server wants to incentivize enough new clients to move the dynamic towards a higher self-fulfilling expectation equilibrium. 
Let client $i$'s utility function be $U_i(n)$ when the number of samples in the coalition is $n$. Then, we can define the mapping $h(\cdot)$ for this setting as the total number of samples from clients with higher expected utility gain than personal cost, \ie
\begin{align*}
h(n) = \sum_{i: U_i(n)\ge c_i} n_i.
\end{align*}
Suppose the initial coalition is empty, \ie $N_S^{(0)}=0$ and let $n_{next}$ be the next tipping point with $h(n_{next}) \ge n_{next}$, that is
\begin{align*}
n_{next} = \arg\min\{j: j> N_S^{(t)}, h(j)\ge j\}.
\end{align*}
Then, our goal in the current time step is to move to $n_{next}$ with the minimum payment. Let $\overline C=\{i: U_i^{t}(N_S)<c_i\}$ be the set of clients who are not already in the coalition at the current time step $t$. Then, we can solve an optimization problem to move to $n_{next}$ as
\begin{align*}
\min & \sum_{i: i\in \overline C} x_i \cdot (c_i - U_i^{(t)}(N_S))\\
\text{s.t} & \sum_{i:i\in\overline C} x_i \cdot n_i \ge n_{next}- N^{(t)}_S.
\end{align*}
where $x_i\in \{0,1\}$. This is a knapsack problem that can be solved by dynamic programming. This process is repeated until we have reached the highest self-fulfilling expectation equilibrium. 
\section{Discussion}
We have initiated the study of clients' behaviors in FL under a network effects game model, where each client's participation decision changes as the network grows over time. We showed that whenever the entire population shares an expectation on the coalition size, we can  characterize the dynamics of the client's participation game for the mean estimation problem and a more general setting with a utility oracle. We identified self-fulfilling expectation equilibrium points, where the coalition after all clients have computed the best-response strategy matches the shared expectation. Moreover, we showed that without explicit intervention from the server, the client's participation dynamic naturally converges to self-fulfilling expectation equilibria. Based on this characterization, we proposed a cost-efficient payment scheme that incentivizes clients so that the final coalition reaches a desirable equilibrium. 

Our framework offers  insights into  clients' participation dynamics in FL. A limitation of our work is that we did not consider more realistic learning problems, and also assumed assumed that all clients are trying to solve a global problem instead of a personalized approach. Despite these limitations, we believe that our work provides a meaningful contribution to the study of incentives in FL, and  hope that continuing to understand client participation dynamics can lead to the design of explicit incentive mechanisms that induce desirable final outcomes in practical FL applications. 

\section*{Acknowledgement}
ZSW, DN were supported in part by the NSF FAI Award \#1939606, NSF SCC Award \#1952085, a Google Faculty Research Award, a J.P. Morgan Faculty Award, a Facebook Research Award, and a Mozilla Research Grant. Any opinions, findings, conclusions, or recommendations expressed in this material are those of the authors and not necessarily reflect the views of the National Science Foundation and other funding agencies.

\bibliographystyle{abbrvnat}
\bibliography{bib}
\newpage
\appendix
\section{Appendix: Utility Gain in Mean Estimation Problem}
To derive the utility gain from joining a coalition, we first need to derive the expected error (MSE) for any client $i$. Here, the expectation in MSE is taken over random draw of $\mu_i \sim \calN(\theta, \sigma^2_\theta)$, random draw of local data from its distribution, and random draw of test points.  

First, we derive the local MSE for client $i$ if they do not join the coalition:
\begin{lemma}[Local MSE]
The local expected MSE for a client with $n_i$ samples is:
\begin{equation}
    MSE(w_i) = \frac{1}{n_i} + \sigma^2_\theta
\end{equation}
\end{lemma}
\begin{proof}
By definition of MSE, we have:
\begin{align}
    &\quad MSE(w_i)\\
    &= \E_{\substack{\mu_i \sim \calN(\theta,\sigma^2_\theta)\\ X_i \sim \calN(\mu_i, 1)}}[(w_i - \theta)^2]\\
    &= \E_{\substack{\mu_i \sim \calN(\theta, \sigma^2_\theta) \\ X_i \sim \calN(\mu_i, 1)}}[[(w_i - \mu_i) + (\mu_i - \theta)]^2]\\
    &= \E_{\substack{\mu_i \sim \calN(\theta,\sigma^2_\theta)\\ X_i \sim \calN(\mu_i,1)}}[(w_i - \mu_i)^2] + \E_{\mu_i \sim \calN(\theta,\sigma^2_\theta)}[(\mu_i - \theta)^2] + 2\E_{\substack{\mu_i \sim \calN(\theta,\sigma^2_\theta)\\ X_i \sim \calN(\mu_i,1)}}[(w_i - \mu_i)(\mu_i - \theta)]\\
    &= \frac{1}{n_i} + \sigma_{\theta}^2 + 2\E_{\substack{\mu_i \sim \calN(\theta,\sigma^2_\theta)\\X_i \sim \calN(\mu_i,1)}}[(w_i - \mu_i)(\mu_i - \theta)]
\end{align}
where the bound on first term comes from \citet{modelsharing2020donahue} and the second bound comes from definition of variance. For the third term, since the expectation is taken over $X_i$ and $\mu_i$, we have:
\begin{align}
&\quad 2\E_{\substack{\mu_i \sim \calN(\theta,\sigma^2_\theta) \\ X_i \sim \calN(\mu_i, 1)}}\left[(w_i - \mu_i)(\mu_i - \theta)\right]\\
&= \E_{\mu_i \sim \calN(\theta,\sigma^2_\theta)}\left[\E_{X_i \sim \calN(\mu_i, 1)}[(w_i - \mu_i)(\mu_i - \theta) | \mu_i]\right]\\
&= \E_{\mu_i \sim \calN(\theta,\sigma^2_\theta)}\left[(\mu_i - \theta) \E_{X_i \sim \calN(\mu_i,1)}[w_i - \mu_i | \mu_i]\right]\\
&= \E_{\mu_i \sim \calN(\theta,\sigma^2_\theta)}\left[ (\mu_i - \theta) (\mu_i - \mu_i) \right]\\
&= 0
\end{align}
Hence, the local MSE for client $i$ is:
\begin{align}
    MSE(w_i) = \frac{1}{n_i} + \sigma_{\theta}^2
\end{align}
\end{proof}
If client $i$ join the coalition $S$, instead of using their own local model, they will use the shared model as an estimate for the global parameter $\theta$. Note that we do not consider any form of personalization in this section as everyone in the coalition $S$ share the same estimator. Hence, the global MSE for client $i \in S$ is:
\begin{lemma}[Coalition MSE]
The global expected MSE for a client with $n_i$ samples in coalition $S$ is:
\begin{equation}
	MSE(w_S) = \frac{1}{N_S} +  \left(\frac{\sum_{i\neq j}n_j^2 + (N_S - n_i)^2}{N_S^2} + \frac{2n_i}{N_S}-1\right)\sigma_{\theta}^2
\end{equation}
where $N_S = \sum_{j \in S} n_j$ is the total number of samples in the coalition.
\end{lemma}
\begin{proof}
By definition of MSE, we have:
\begin{align}
	&\quad MSE(w_S)\\
	&= \E_{\substack{\mu_i \sim \calN(\theta,\sigma^2_\theta)\\X_j \sim \calN(\mu_j,1), \forall j \in S}}[(w_S - \theta)^2]\\
	&= \E_{\substack{\mu_i \sim \calN(\theta,\sigma^2_\theta)\\X_j \sim \calN(\mu_j,1), \forall j \in S}}[[(w_S - \mu_i)+(\mu_i - \theta)]^2]\\
	&= \E_{\substack{\mu_i \sim \calN(\theta,\sigma^2_\theta)\\X_j \sim \calN(\mu_j,1), \forall j \in S}}[(w_S - \mu_i)^2] + \E_{\mu_i \sim \calN(\theta,\sigma^2_\theta)}[(\mu_i - \theta)^2]\\
	&\quad + 2\E_{\substack{\mu_i \sim \calN(\theta,\sigma^2_\theta)\\X_j \sim \calN(\mu_j,1), \forall j \in S}}[(w_S - \mu_i)(\mu_i - \theta)]\\
	&= \E_{\substack{\mu_i \sim \calN(\theta,\sigma^2_\theta)\\X_j \sim \calN(\mu_j,1), \forall j \in S}}[(w_S - \mu_i)^2] + \sigma_{\theta}^2 + 2\E_{\substack{\mu_i \sim \calN(\theta,\sigma^2_\theta)\\X_j \sim \calN(\mu_j,1), \forall j \in S}}[(w_S - \mu_i)(\mu_i - \theta)]\\
	&= \frac{1}{N_S} +  \left(\frac{\sum_{i\neq j}n_j^2 + (N - n_i)^2}{N_S^2} + 1\right)\sigma_{\theta}^2 + 2\E_{\substack{\mu_i \sim \calN(\theta,\sigma^2_\theta)\\X_i \sim \calN(\mu_i,1)}}[(w_S - \mu_i)(\mu_i - \theta)]
\end{align}
    where the first term comes from \citet{modelsharing2020donahue}. Note that the expectation is taken over random draws of mean parameter $\mu_j$ and local samples for all clients $j$ in coalition $S$. Hence, for the third and final term, we have:
\begin{align}
    &\quad 2\E_{\substack{\mu_i \sim \calN(\theta,\sigma^2_\theta)\\X_j \sim \calN(\mu_j,1), \forall j \in S}}[(w_S - \mu_i)(\mu_i - \theta)]\\
    &= 2\E_{\mu_i \sim \calN(\theta, \sigma_{\theta}^2)}\left[ \E_{X_j \sim \calN(\mu_j, 1), \forall j \in S}[(w_S - \mu_i)(\mu_i - \theta) | \mu_i] \right]\\
    &= 2\E_{\mu_i \sim \calN(\theta,\sigma^2_\theta)}\left[ (\mu_i - \theta) \E_{X_j \sim \calN(\mu_j,1), \forall j \in S}[w_S - \mu_i | \mu_i] \right]\\
    &=  2\E_{\mu_i \sim \calN(\theta,\sigma^2_\theta)}\left[ (\mu_i - \theta) (\E_{X_j \sim \calN(\mu_j,1), \forall j \in S}[w_S| \mu_i] - \mu_i )\right]\\
    &=  2\E_{\mu_i \sim \calN(\theta,\sigma^2_\theta)}\left[ (\mu_i - \theta) \left(\E_{X_j \sim \calN(\mu_j,1), \forall j \in S}\left[\frac{1}{N} \sum_{j \in S} w_j n_j \Bigg| \mu_i \right] - \mu_i \right)  \right]\\
    &= 2\E_{\mu_i \sim \calN(\theta,\sigma^2_\theta)} \left[ (\mu_i - \theta) \left( \frac{\sum_{j\in S} \mu_j n_j}{N} - \mu_i \right) \right]\\
    &= 2\E_{\mu_i \sim \calN(\theta,\sigma^2_\theta)} \left[(\mu_i - \theta) \left( \frac{\sum_{j\in S, j\neq i} n_j (\mu_j - \mu_i)}{N} \right) \right]\\
    &= \frac{2}{N} \E_{\mu_i \sim \calN(\theta,\sigma^2_\theta)} \left[(\mu_i - \theta)\left(\sum_{j\in S, j\neq i} n_j (\mu_j - \mu_i)\right) \right]
\end{align}
Take expectation over all $\mu_j$
\begin{align}
    &\quad \frac{2}{N} \E_{\mu_i,\mu_j \sim \calN(\theta,\sigma_{\theta}^2),j\in S,j\neq i} \left[(\mu_i - \theta)\left(\sum_{j\in S, j\neq i} n_j (\mu_j - \mu_i)\right) \right]\\ &= \frac{2}{N}\sum_{j\in S,j\neq i}n_j\E_{\mu_i,\mu_j \sim \calN(\theta,\sigma_{\theta}^2),j\in S,j\neq i} \left[(\mu_i-\theta)(\mu_j-\mu_i)\right]\\
    &=\frac{2}{N}\sum_{j\in S,j\neq i}n_j\E_{\mu_i,\mu_j \sim \calN(\theta,\sigma_{\theta}^2),j\in S,j\neq i} \left[\mu_i(\mu_j-\mu_i)\right]\\
    &= \frac{2}{N}\sum_{j\in S,j\neq i}n_j(\theta^2-\theta^2-\sigma_{\theta}^2)\\
    &=-\frac{2(N-n_i)}{N}\sigma_{\theta}^2
\end{align}
Substituting this expression into the $MSE(w_S)$ calculation, we obtain:
\begin{align}
    MSE(w_S) &= \frac{1}{N_S} +  \left(\frac{\sum_{i\neq j}n_j^2 + (N_S - n_i)^2}{N_S^2} + \frac{2n_i}{N_S}-1\right)\sigma_{\theta}^2
\end{align}
\end{proof}
Finally, we can derive the utility gain of client $i$ from joining the coalition. Formally, the utility gain $u_i$ for client $i$ is the difference in mean squared error between not joining the coalition (local MSE) and joining the coalition (coalition MSE):
\begin{lemma}[Utility gain] 
The utility gain of client $i$ from joining the coalition and using the shared model $w_S$ is:
\begin{equation}
    u(\{n_j\}_{j\in S}, i) = -\frac{1}{N_S}+\frac{1}{n_i}-\left(\frac{\sum_{i\neq j}n_j^2 - (N_S - n_i)^2}{N_S^2}-\frac{2(N_S-n_i)}{N_S}\right)\sigma_{\theta}^2
\label{eq:utility-gain}
\end{equation}
where $N_S = \sum_{j \in S} n_j$.
\end{lemma}

\paragraph{Homogeneous setting} When every client in the population has the same number of samples $n = n_i \forall i \in [M]$, we can simplify the utility gain formulation above to be a function of the number of clients in the coalition $K$ and the number of local samples $n$. 

\begin{lemma}[Utility gain under fixed $n_i$]
Assume that $n_i=n$ for all $i$. The utility gain of client $i$ from joining the coalition with $K$ clients and using the shared model $w_S$ is:
\begin{equation}
    u(K,n) = \frac{K-1}{Kn} + \frac{3K^2 - 5K + 2}{K^2}  \sigma^2_{\theta}
\end{equation}
\end{lemma}
\begin{proof}
Substitute $n = n_i, \forall i$ and $N_S = Kn$ in the utility gain expression in \Cref{eq:utility-gain}, we have:
\begin{align}
    u(K,n,i) &= -\frac{1}{N_S}+\frac{1}{n_i}-\left(\frac{\sum_{i\neq j}n_j^2 - (N_S - n_i)^2}{N_S^2}-\frac{2(N_S-n_i)}{N_S}\right)\sigma_{\theta}^2\\
    &= \frac{-1}{Kn} + \frac{1}{n} - \left(\frac{(K - 1)n^2 - (Kn - n)^2}{(Kn)^2} - \frac{2(Kn - n)}{Kn} \right)\sigma^2_{\theta}\\
    &= \frac{K - 1}{Kn} - \left( \frac{(K-1)n^2 - K^2n^2 + 2Kn^2 - n^2}{K^2n^2} - \frac{2K - 2}{K} \right)\sigma^2_{\theta}\\
    &= \frac{K-1}{Kn} - \left(\frac{K - 1 - K^2 + 2K - 1}{K^2} - \frac{2K - 2}{K} \right)\sigma^2_{\theta}\\
    &= \frac{K-1}{Kn} + \frac{3K^2 - 5K + 2}{K^2}  \sigma^2_{\theta}
\end{align}
\end{proof}

\begin{figure}[ht]
    \centering
    \includegraphics[scale=0.5]{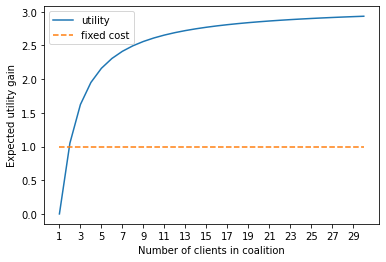}
    \caption{Utility gain with equal contribution}
    \label{fig:utility-gain}
\end{figure}

\section{Appendix: FL in Mean Estimation Problem}
In this example, suppose different clients have different costs of joining the coalition and the same number of local samples. W.L.O.G, we order the clients by their cost in ascending order. We also assume the cost of joining the coalition for the $i$-th client is $c_{\min} + c(i)$, where $c_{\min}$ is a fixed positive cost for communication and $c(i)$ is a positive increasing function. 

Since all clients joining the coalition will share the same model, we have $u = u_i$ for all $i \in [M]$. Let the equilibrium cost $c_{\min} > 0$ and the expected number of clients in coalition $K$ form a self-fulfilling expectation equilibrium, such that $c_{\min} + c(K) = u(K,n)$.

We can write the new utility gain function for the $i$-th client as:
\begin{equation}
    f(K, n, i) = u(K,n) - c(i)
\end{equation}

If all clients share a common belief that $K$ clients will join the coalition, then client $i$ would join the coalition if $f(K, n, i) \geq c_{\min}$. Hence, if any client at all would join, the set of clients joining will be between $0$ and $\hat{K}$, where $\hat{K} = \inf \{ x \in \mathbb{R}: f(K,n,i) \geq c_{\min} \}$. Hence, to find the equilibrium client $\hat{K}$, we proceed to solve the following equation: 
\begin{equation}
    f(K, n, \hat{K}) = c_{\min} 
\end{equation}
Fixing the number of local data $n$, this is equivalent to \begin{equation}
    \hat{K} = f^{-1}(c_{\min}) = c^{-1}(u(K,n) - c_{\min})
\end{equation}

\subsection{Existence of an equilibrium}
Let $h$ be the mapping from $K$ to $\hat K$, meaning that when all clients share a common belief $K$ clients will join the coalition, then the number of clients willing to join will be $\hat K = h(K)$. 

\paragraph{Proof of \Cref{thm:sufficient-condition-homogeneous}}
Suppose there exists $K$ with $h(K) \ge K$, we prove that there must exist $K^* \ge K$ with $h(K^*) = K^*$. First if $h(K)=K$, we have found $K^*=K$. So we only need to consider the case $h(K)>K$. 
When $h(K)>K$, we prove the following claim: we can find $K'\ge K$ with 
\begin{align} \label{eqn:condition}
    h(K')\ge K' \text{ and } h(h(K'))\le h(K'). 
\end{align}

We first set $K'=K$ and check whether $K'$ satisfies the condition~\eqref{eqn:condition}. If the condition is not satisfied, we set $K'$ to be its function value $h(K')$ and iterate until the condition is satisfied. We can prove that $K'$ must strictly increase at each iteration. This is because at the beginning we have $h(K')>K'$ and we only continue the iteration when $h(h(K'))>h(K')$. Therefore, the iteration must end because $K'$ must strictly increase at each iteration and the value of $K'$ cannot exceed the number of clients $M$. And when the iteration ends, we either find  $K'$ with~\eqref{eqn:condition}, or we find an equilibrium $K'=h(K')=M$.

Finally, we prove that for $K'$ with~\eqref{eqn:condition}, we have $h(h(K')) = h(K')$ so we find an equilibrium $K^* = h(K')$. We use the fact that $h(\cdot)$ must be a non-decreasing function. So we have $h(h(K'))\ge h(K')$ for $h(K') \ge K'$. Together with $h(h(K'))\le h(K')$ in~\eqref{eqn:condition}, $h(h(K')) = h(K')$.

This result can also be interpreted as, every time $\hat K$ goes above $K$, we must have an equilibrium. 

\paragraph{Equivalent Sufficient Condition}
In our formulation, the function $h(K)$ is a composition of the cost function $c(i)$ and the utility function $u(K,n)$. Formally, we have $h(K) = c^{-1}(u(K,n))$. Note that we assume the cost function $c(i)$ is monotonically increasing with regard to the client index $i$. Hence, we can apply $c(.)$ to both side of the sufficient condition in \Cref{thm:sufficient-condition-homogeneous} and get:
\begin{align*}
    h(K) &\geq K\\
    \iff c^{-1}(u(K,n)) &\geq K\\
    \iff c(c^{-1}(u(K,n))) &\geq c(K)\\
    \iff u(K,n) &\geq c(K)
\end{align*}
Hence, the sufficient condition for an equilibrium to happen can be rewritten as: a self-fulfilling expectation equilibrium exists as long as there exists some $K$ with $u(K,n) \geq c(K)$. 


Then, in the special case above, there exists only one non-trivial stable equilibrium at $K = M$. 

\paragraph{Proof of \Cref{thm:sufficient-condition-stable-homogeneous}}
We follow the same analysis as in the proof of \Cref{thm:sufficient-condition-homogeneous} and apply the definition of stable equilibrium in \Cref{def:stable-equilibrium}.   

\paragraph{Proof of \Cref{cor:convergence-homogeneous}} Suppose that the initial expectation $K$ is a self-fulfilling expectation equilibrium. We want to look at what happens when there is a slight perturbation in the actual number of people joining the coalition $\hat{K} = K$. 
There are two cases:
\begin{itemize}
    \item If $K$ is a stable equilibrium: By definition, any slight perturbation to $\hat{K}$ will make the expectation converges back to $K$.
    \item If $K$ is not a stable equilibrium: When the actual outcome $\hat{K}$ changes slightly and the shared expectation $K$ stays the same, the utility gain $u(K,n)$ does not changes. We look at how the cost of joining $c(\hat{K})$ changes when a client leaves or when a new client joins the coalition. 
    \begin{itemize}
        \item If the $K+1$-th client joins the coalition: the cost of joining for the $K+1$-th client is $c(K+1) < u(K,n)$. Since we assume the cost function is monotonically increasing (non-decreasing), the fact that the $K+1$-th client joining does not change the opt-in/opt-out decision of all previous clients. Also, since the utility function $u(K,n)$ is non-decreasing, then $u(K+1, n) > u(K,n) > c(K+1)$. Hence, the shared expectation will become $K+1$. When the next client $K+2$ decides whether to join the coalition, they will compare their personal cost $c(K+2)$ to the new shared expectation $K+1$. This process continues until we reach a new equilibrium $K'$.
        \item If the $K$-th client leaves the coalition: the $K$-th client leaves the coalition when $c(K) > u(K,n)$. Since the cost function is monotonically increasing, we have $c(K+1) > c(K) > u(K,n)$. Hence, all clients from $K+1$ to $M$ will not join the coalition. Then, the maximum number of clients in a coalition is less than $K$ and there is downward pressure on the expectation. Suppose the expectation is now reduced by $1$, i.e. the utility gain is $u(K-1, n)$. There are two cases. If the personal cost of client $K-1$ is less than the utility, i.e. $c(K-1) < u(K-1,n)$, then we arrive at a self-fulfilling expectation equilibrium and no other clients would leave the coalition. Otherwise, if we have $c(K-1) > u(K-1,n)$, then the $K-1$-th client would also leave the coalition. The expectation goes into a downward spiral until we reach some equilibrium $K'$ such that $c(K') < u(K',n)$ and $c(K'+1) > u(K'+1, n)$.
    \end{itemize}
\end{itemize}
\paragraph{Behavior of tipping points} From the previous analysis, we observe that a slight perturbation to the actual outcome at the tipping point (unstable equilibrium) will lead the expectation to move toward either a higher or lower equilibrium. If the coalition settles at a new equilibrium that is also a tipping point, then additional perturbations to the actual outcome will move the expectation to another equilibrium. This behavior only stops when the shared expectation reaches a stable equilibrium.


\paragraph{Proof of \Cref{thm:sufficient-condition-heterogeneous}}
Suppose there exists a number of samples in the coalition $h(K) \geq K$, we prove that there must exist some equilibrium $K^* \geq K$ with $h(K^*) = K^*$. If $h(K) = K$, then $K^* = K$ is the equilibrium. Else, if $h(K) > K$, we can find some $K' \geq K$ with 
\begin{equation}
    h(K') \geq K' \quad \text{and} \quad h(h(K')) \leq h(K')
\label{eq:heterogeneous-condition}
\end{equation}
First, we set $K' = K$ and check whether $K'$ satisfies the condition \cref{eq:heterogeneous-condition}. If the condition is not satisfied, we set $K' = h(K')$ and repeat the process to find $h(h(K'))$. Initially, we have $h(K') > K'$. By design, the inferred number of samples in the coalition $t'$ is always larger than or equal to the expectation $K'$. For each client in the inferred coalition, they can check if their utility gain is larger than the cost of joining. The clients with a higher cost than utility gain will not be in the coalition. There are two cases:
\begin{itemize}
    \item If $h(h(K'))> h(K')$: then we continue for another iteration.
    \item Otherwise, if $h(h(K)) \leq h(K')$: then we have found the $K'$ that satisfy condition \Cref{eq:heterogeneous-condition}.
\end{itemize}
Since there are $M$ clients in the population, we know that there are at most $N_M$ samples in the coalition (where $N_M$ is the total number of samples for every client in the population). Since we only continue the iterations when we have $h(x) > x$ for some number of samples $x$, we know that the iteration is guaranteed to end when every client join the coalition. Hence, we can always find some $K'$ that satisfy \Cref{eq:heterogeneous-condition}.

To complete the proof, we prove that for a number of samples $K'$ that satisfy \Cref{eq:heterogeneous-condition}, we have $h(h(K')) = h(K')$ so we have found an equilibrium $K^* = h(K')$. Since the $h(\cdot)$ function is a non-decreasing function by assumption, we have $h(h(K')) \geq h(K')$ for $h(K') \geq K'$. Combine this fact with \Cref{eq:heterogeneous-condition}, we have $h(h(K')) = h(K')$. 

\paragraph{Necessary condition: non-decreasing $h(\cdot)$} Let $K$ and $K'$ be two expected number of samples in the coalition such that $K > K'$. Also, let $S$ and $S'$ denote the coalition with expectation $K$ and $K'$, respectively. Observe that since $K > K'$, if a client $i$ is in the inferred coalition with expectation $K'$, they would also be in the inferred coalition with expectation $K$ since their utility gain in $S$ is larger than utility gain in $S'$ (which is larger than their personal cost). For some clients $i$ in the inferred coalition $S'_\infer$ with expectation $K'$, there are two cases to consider:
\begin{itemize}
    \item If $u_i \geq c_i$, then client $i$ would be in the coalition $S'$. Since the ordering is maintained, client $i$ would also be in coalition $S$ with expectation $K$. 
    \item If $u_i < c_i$, then client $i$ would not be in the coalition $S'$. Since we order the clients in descending order of $z$, all other clients $j \in [i, M]$ would also not join the coalition $S'$. However, since $K > K$, the expected utility gain for joining with expectation $K$ is greater than the expected utility gain for joining with expectation $K'$. Hence, there is still a possibility that client $i$ would join coalition $S$ with expectation $K$. 
\end{itemize}
Therefore, the number of clients in coalition $S$ is always greater than or equal to the number of clients in coalition $S'$. Similarly, the number of samples in $S$ is greater than or equal to the number of samples in $S'$. Hence, $h(K)$ is a non-decreasing function in $K$. 

\paragraph{Proof of \Cref{cor:convergence-heterogeneous}} Suppose that the expected number of samples in the coalition $K$ is a self-fulfilling expectation equilibrium. In the following paragraph, we describe the network dynamic when there is a slight perturbation to the actual number of samples in the coalition. Note that by definition of self-fulfilling expectation equilibrium, initially we have $\hat{K} = K$. There are two cases to consider:
\begin{itemize}
    \item If $K$ is a stable equilibrium: by definition of stable equilibrium, any slight perturbation to the actual number of samples $\hat{K}$ will lead the expectation converges back to $K$. 
    \item If $K$ is a tipping point: When the actual number of samples $\hat{K}$ changes, the client can update their shared belief and observe that the utility gain is also changed. According to our assumption, the cost of each client is a random variable independent of how many clients are actually in the coalition. Hence, there could be some clients that change their decision based on the updated utility gain. 
\end{itemize}
Specifically, we can take a look at the dynamic at the tipping point when there are $\epsilon$ additional samples or fewer samples in the coalition:
\begin{itemize}
    \item If there are $\epsilon > 0$ more samples in the coalition from some client $j$ joining: the cost of being in the coalition for a client $i \in S$ is $c_i < u(K, i)$. By assumption, the utility gain function is monotonically non-decreasing. Hence, having more samples in the coalition would increase the utility gain for all clients $i \in S$. Consider a client $j$ who was not previously in coalition $S$ has index $j > i: \forall i \in S$, their overall gain from joining the coalition $z_j < z_i: \forall i \in S$. Hence, if client $j$ join the coalition $S$, then all other clients $i \in S$ would still stay in coalition $S$. Then, the shared expectation would now be $K + \epsilon$, and the expected utility gain for all clients is increased. A new client $\ell$ who is not in $S$ would compare this new utility gain $u(K + \epsilon, \ell)$ with their personal cost $c_\ell$ and join if $u(K + \epsilon, \ell) \geq c_\ell$. This process will continue until the population reaches a stable equilibrium, where having slightly more data does not change the number of clients and samples in the coalition. 
    \item If there are $\epsilon>0$ more samples in the coalition from some client $j$ joining and some other client $\ell \in S$ leaving: since we rank the clients by their personal gain $z_i$ with arbitrary tie-breaking, there could be a case where a client $\ell \in S$ and client $j \notin S$ have $z_\ell = z_j$ and $n_\ell < n_j$. Then, we can replace client $\ell$ by client $j$ in the coalition and gain $\epsilon = n_j - n_\ell$ samples. The expected utility gain from joining the coalition increases since there are samples in $S$. Since the expected number of samples changes, the ordering of clients in the population also changes. However, similar to our previous argument in the necessary condition, the clients who are already in the coalition should still be in the coalition after having more samples. The clients who are not in the clients can recompute their expected utility gain and decide whether to join the coalition or not. Since $K$ is a tipping point, there is at least one client who is now willing to join the coalition. This process only end when the population converges to a stable equilibrium, where having more samples does not change the composition of the coalition.  
    \item If there are $\epsilon$ fewer samples in the coalition from some client $j$ leaving: since we order clients by their personal gain $z_i$, all other clients $\ell \in [j, M]$ will not join the coalition. Then, the number of samples in the coalition is fewer than $K$, and the expected utility gain is reduced. All clients in the population can recompute their utility gain according to the updated number of samples. Clients will keep leaving the coalition if their utility gain is less than their personal cost until the population is at a stable equilibrium. 
\end{itemize}

\section{Utility Oracle Appendix}
In the following analysis, we assume a more general form of the utility function. Let $U_i^{(t)} = u(N_S^{(t)})$ be the utility gain for a client $i$ at round $t$ if client $i$ is in the coalition with $N_S^{(t)}$ samples. This formulation of utility only depends on the size of the coalition and the opt-in/opt-out decisions of other clients in the coalition. 
Since all clients know the true utility regardless of whether or not they join the coalition, there is no longer a need for the population to form an expectation over the size of the coalition. In this setting, the expected utility for time step $t$ is the oracle utility of the previous time step $t-1$, \ie $\E[U_i^{(t)}] = U_i^{(t-1)}$. 

\paragraph{Example: Participation Dynamic with Utility Oracle} Suppose there are four clients in the population whose local sample sizes follow a uniform distribution. That is, if the total number of samples in the coalition is $N$, then client $1$ has $1/N$ samples, client $2$ has $2/N$ samples, client $3$ has $3/N$ samples, and client $4$ has $4/N$ samples. Furthermore, assume that each client $i$ has a fixed cost $c_i$ and $0 \approx c_1 \leq c_2 \leq c_3 \leq c_4$. 
\begin{itemize}
    \item First coalition is $S = \{1,4\}$: In the beginning, assume that client $1$ and $4$ are in the coalition: $S = \{1,4\}$ with $5/N$ samples. The oracle utility at time step $1$ is $U^{(1)}(5/N)$. Suppose at the start of the time step $2$, client $4$ observes that $U^{(1)}(5/N) < c_4$ and decides to leave. We have three events that can happen at this step:
\begin{enumerate}
    \item Both client $2$ and $3$ join: If client $2$ and $3$ observe that the expected utility from previous time step $U^{(1)}(5/N) \geq c_3 > c_2$, then they will join the coalition. Then, at the end of time step $2$, the coalition is $S = \{1,2,3\}$ with utility $U^{(2)}(6/N)$. At time step $3$, if client $4$ observes that $U^{(2)}(6/N) < c_4$, then they would still opt-out of the coalition. Otherwise, if we have $U^{(2)}(6/N) \geq c_4$, then client $4$ would join the coalition instead. For clients $1,2$ and $3$, they do not leave the coalition since $U^{(2)}(6/N) \geq U^{(1)}(5/N) \geq c_3 > c_2 > c_1$. Hence, at the end of the time step $3$, the coalition is $S = \{1,2,3,4 \}$ and utility is $U^{(3)}(N/N)$. In either case, we arrive at an equilibrium where no clients have the incentive to change their decisions. 
    
    \item Only client $2$ joins: If we have $c_3 > U^{(1)}(5/N) \geq c_2$, then only client $2$ would join the coalition at time step $2$. At the end of time step $2$, the coalition is $S = \{1,2\}$ with utility $U^{(2)}(3/N)$. At time step $3$, we have $U^{(2)}(3/N) < U^{(1)}(5/N) < c_3 < c_4$, so client $3$ and $4$ would not join the coalition. If client $2$ observes that $U^{(2)} < c_2 \leq c_2$, then client $2$ would also leave the coalition. At the end of time step $3$, the coalition is $S = \{1\}$ with utility $U^{(3)}(1/N)$. Otherwise, if we have $U^{(2)} > c_2$, then client $2$ would still stay in the coalition. In this case, the coalition utility at the end of time step $3$ is $U^{(3)}(3/N) = U^{(2)}(3/N)$ and we have arrived at an equilibrium.
    
    \item No other client joins: If we have $c_3 > c_2 > U^{(1)}(5/N)$, then no client would join the coalition. The coalition at the end of day $2$ is $S = \{1\}$ with utility $U^{(2)}(1/N)$. Due to the monotonicity assumption, we have $U^{(2)}(1/N) \leq U^{(1)}(5/N) < c_2 < c_3 < c_4$, and no other clients would join the coalition. Regardless of  the client $1$'s decision, we will arrive at an equilibrium.
\end{enumerate}
Therefore, in any of the aforementioned events, we will arrive at an equilibrium. 
\end{itemize}

In our prior analysis, we use $h(K)$ to denote the actual size of the coalition when the shared size is $K$. When the clients can observe the oracle utility, they adjust the shared expectation for time step $t$ to be the oracle utility from time step $t-1$. Hence, if $K$ is the coalition's size at round $t$, then we use $h(K)$ to denote the size of the coalition at round $t+1$. 


\paragraph{Proof of \Cref{thm:utility-oracle-convergence}}
Suppose there exists an arbitrary coalition containing $K_1$ samples formed by some clients at time step $t=1$. In the following analysis, we show the iterative process for the dynamic to converge to a self-fulfilling expectation equilibrium. 
\paragraph{At time step $2$:} Suppose there exist some clients $i \in S$ with $U^{(1)}_i(K_1) < c_1$. Then, at the start of time step $2$, these clients would leave the coalition with $N^{(2)}_L \geq 0$ samples. If no clients leave the coalition at this point, then $N^{(2)}_L = 0$, otherwise $N^{(2)}_L > 0$. At the same time, some other clients $j \in J_2: U_j^{(1)}(K_1) > c_j$ would join the coalition with $N_J^{(2)} \geq 0$ samples. Then, at the end of time step $2$, the total size of coalition $S$ is $K_2 = K_1 - N_L^{(2)} + N_J^{(1)}$. 
\paragraph{At time step $3$:} At the beginning of time step $3$, clients $i \in S$ first re-evaluate their opt-in decision by comparing their cost with the expected utility.
\begin{itemize}
    \item If we have $K_2 \geq K_1$, then by monotonicity assumption, we have $\forall i \in S: U_i^{(2)}(K_2) \geq U_i^{(1)}(K_1) > c_i$. That is, all clients who are already in the coalition gain more utility by staying and not changing their decision. Hence, the size of the coalition would not decrease at time step $3$.
    
    \item If we have $K_2 < K_1$, then by monotonicity assumption, there are some clients $i$ who will leave the coalition due to insufficient utility gain: $\exists i \in S: U_i^{(2)}(K_2) \leq c_i < U_i^{(1)}(K_1)$. Note that for other clients $j \notin S$, they observe $U_j^{(2)}(K_2) < U_j^{(1)}(K_1) < c_j$ so they would still not join the coalition at this time step. Hence, the size of the coalition would decrease at time step $3$.   
\end{itemize}
Observe that if we have $K_2 \geq K_1$, then the size of the coalition does not decrease at time step $3$, \ie $K_3 \geq K_2$. On the other hand, if $K_2 < K_1$, then we will have $K_3 < K_2$. 

Induction hypothesis: If $K_t < K_{t-1}$, then $K_{t+1} \leq K_t$. Otherwise, if $K_t \geq K_{t-1}$, then $K_{t+1} \geq K_t$.

Show that the induction hypothesis holds for time step $t+1$: 
\begin{itemize}
    \item If the number of samples in the coalition decreases at time step $t$, then we have $K_t < K_{t-1}$. By monotonicity assumption, for all clients $i \in [M]$, we have $U_i^{(t)}(K_t) < U_i^{(t-1)}(K_{t-1})$. Then at the start of time step $t+1$, the clients who have not joined the coalition in the previous time step would also not join at time step $t+1$ due to insufficient utility gain. On the other hand, clients $j\in S$ who are in the coalition at time step $t$ might leave the coalition if $U_j^{(t-1)}(K_{t-1}) \geq c_j > U_j^{(t)}$. Hence, the size of the coalition would either decrease at time step $t+1$ or stay the same (where we have found an equilibrium). 

    \item If the number of samples in the coalition does not decrease at time step $t$, then we have $K_t \geq K_{t-1}$. By monotonicity assumption, for all clients $i \in [M]$, we have $U_i^{(t)}(K_t) \geq U_i^{(t-1)}(K_{t-1})$. Hence, at the start of time step $t+1$, all clients $j \in S$ who are already in the coalition gain more utility by staying and not changing their decision. Also, there could be some clients $\ell$ previously not in the coalition with $U_{\ell}^{(t)}(K_t) \geq c_{\ell} \geq U_{\ell}^{(t-1)}(K_{t-1})$ who will join at time step $t+1$. Thus, the size of the coalition would either stay the same (where we have found an equilibrium) or increase at time step $t+1$. 
\end{itemize}
This iterated process will continue until either we have the same coalition size in two consecutive rounds (which is a self-fulfilling expectation equilibrium) or we have all clients (or no clients) in the coalition. Therefore, the dynamic with utility oracle will converge to a self-fulfilling expectation equilibrium. 

\begin{remark}
Compared to our previous analysis of the heterogeneous setting in a network effect game, we no longer need to form an inferred coalition to form an expectation over the utility gain from joining. Since all clients, regardless of being in the coalition or not, know the true utility, we can instead start the dynamic with any arbitrary coalition of clients. Then, other clients in the population can start making decisions based on the oracle utility.
\end{remark}

\section{Appendix: Incentive with Payment}
\paragraph{Proof of \Cref{thm:payment-homogeneous}}
Initially, the shared expectation is there are no clients in the coalition. If the server does not send any payment message, then there is no client willing to join the coalition and the dynamic is stuck at the trivial stable equilibrium of $0$. On the other hand, if the server sends payments to $K_0$ clients, where $K_0$ is between $0$ and the first tipping point, then the dynamic would converge back to the trivial equilibrium of $0$. Hence, to move the dynamic toward a higher equilibrium, the server needs to first send payments to $K_0$ clients, where $S = \{i\}_{i=1}^{K_0}$ is the first tipping point. To minimize the total payment needed, the server can order the population by their cost in ascending order and send payment to the first $K_0$ clients with the smallest costs. 

By definition, once the population has reached the first tipping point, the expected number of clients is equal to the actual number of clients in the coalition. Hence, no new client would join the coalition without external incentives. At a tipping point, if there is an extra client joining the coalition then the dynamic would move toward a higher equilibrium. Thus, the server would send a payment message to the client with the smallest cost who is not in $S$. With this new client joining, the expected utility gain for the next round is increased, and new clients would keep joining the coalition until the dynamic has reached a new equilibrium. If this new equilibrium is also a tipping point, then the server only needs to pay for one new client to join the coalition. On the other hand, if the new equilibrium is a stable equilibrium, then the server needs to either pay $K_1$ new clients, where $K_1$ is the difference between this stable equilibrium and the next tipping point, or $0$ if the coalition has reached the largest stable equilibrium. The number of paid clients at this step cannot be less than $K_1$ as the dynamic would converge back to the stable equilibrium.
This payment schedule would continue until the dynamic converges to the largest equilibrium. The total amount of payment is the sum of payments sent by the server at each step. In the case where there are only two stable equilibria (the trivial equilibria at $0$ and at $M$), the total amount of payment scales linearly with the number of tipping points in the dynamic. 


\paragraph{Incentive with payment when there exists a utility oracle:} Note that we assume the expected utility when there exists an oracle only depends on the number of samples in the coalition. Let client $i$'s utility function be $U_i(n)$ when the number of samples is $n$. Then we can define $h(\cdot)$ function as
\begin{align*}
h(n) = \sum_{i: U_i(n)\ge c_i} n_i.
\end{align*}
Suppose we start with $\tilde n=0$, let $n_{next}$ be the next tipping point with $h(n_{next}) \ge n_{next}$, that is
\begin{align*}
n_{next} = \arg\min\{j: j>\tilde n, h(j)\ge j\}.
\end{align*}
then our goal in the current step is to move to $n_{next}$ with a minimum payment. Let $\overline C=\{i: U_i(\tilde n)<c_i\}$ be the clients who are not in the coalition yet. Then we can solve an optimization problem to move to $n_{next}$ as
\begin{align*}
\min & \sum_{i: i\in \overline C} x_i \cdot (c_i - U_i(\tilde n))\\
\text{s.t} & \sum_{i:i\in\overline C} x_i \cdot n_i \ge n_{next}- \tilde n.
\end{align*}
where $x_i\in \{0,1\}$. This is a knapsack problem and it can be solved by dynamic programming.
\end{document}